\documentclass{article}
\usepackage{spconf} 

\usepackage{amsmath,amssymb,amsfonts}
\usepackage{bm}
\usepackage{graphicx}
\usepackage{algorithmic}
\usepackage{xcolor}
\usepackage{cite}       
\usepackage{url}        
\usepackage[colorlinks=true, linkcolor=blue, urlcolor=blue, citecolor=blue]{hyperref}

\usepackage{amsthm}
\newtheorem{theorem}{Theorem}
\newtheorem{lemma}{Lemma}[section]
\newtheorem{definition}{Definition}
\newtheorem{proposition}{Proposition}[section]
\newtheorem{assumption}{Assumption}[section]

\newtheorem{corollary}{Example}[section]
\newtheorem{claim}{Claim}[theorem]

\DeclareMathOperator{\Tr}{Tr}

\input{my_symbol.sty}

\begin{document}



\title{A Generative 
Model for Controllable Feature Heterophily in Graphs}

\name{Haoyu Wang$^{1}$, Renyuan Ma$^{2}$, Gonzalo Mateos$^{3}$, and Luana Ruiz$^{1}$}

\address{
$^{1}$ Dept. of Applied Mathematics and Statistics, Johns Hopkins University, Baltimore, MD, USA\\
$^{2}$ Dept. of Statistics and Data Science, Yale University, New Haven, CT, USA\\
$^{3}$ Dept. of Electrical and Computer Engineering, University of Rochester, Rochester, NY, USA
}

\maketitle

\begin{abstract}
We introduce a principled generative framework for graph signals that enables explicit control of feature heterophily, a key property underlying the effectiveness of graph learning methods. Our model combines a Lipschitz graphon-based random graph generator with Gaussian node features filtered through a smooth spectral function of the rescaled Laplacian. We establish new theoretical guarantees: (i) a concentration result for the empirical heterophily score; and (ii) almost-sure convergence of the feature heterophily measure to a deterministic functional of the graphon degree profile, based on a graphon-limit law for polynomial averages of Laplacian eigenvalues. These results elucidate how the interplay between the graphon and the filter governs the limiting level of feature heterophily, providing a tunable mechanism for data modeling and generation. We validate the theory through experiments demonstrating precise control of homophily across graph families and spectral filters.
\end{abstract}


\section{Introduction}

The success of many graph information processing problems, including node-level tasks in graph machine learning \cite{zhu2020beyond,kipf17-classifgcnn} and network topology inference \cite{GM6,GM3,GM5}, hinges on the alignment between graph topology and node features, often summarized by the notion of homophily or heterophily. We develop a generative framework for graphs and node features (i.e., graph signals) that allows explicit control of \textit{feature heterophily} in the range from homophily to heterophily. Feature heterophily extends the usual notion of homophily (nodes with the same labels are more likely to establish relational ties) \cite{mcpherson2001birds} to arbitrary feature vectors.
Our model combines graphs sampled from a graphon \cite{lovasz2012large,borgs2017graphons,eldridge2016graphons,airoldi2013stochastic} with features generated by a stationary graph signal model, where white Gaussian noise is transformed by a polynomial graph filter \cite{marques2017stationary}.

Our theoretical contributions are twofold. 
First, we establish concentration of the empirical feature heterophily around a deterministic spectral quantity, showing that randomness in feature generation becomes negligible as the graph grows. 
Second, we prove almost-sure convergence of the heterophily score to a deterministic functional of the graphon degree profile, obtained through a graphon-limit law for polynomials of Laplacian eigenvalues. 
Together, these results reveal a simple and interpretable mechanism for heterophily control in graph generation via a combined graphon and stationary signal model. They further show that, in this combined model, heterophily is governed jointly by the graphon and the polynomial filter giving rise to the features. Notably, the dependence on the graphon is only through its degree function.

We validate both the concentration and convergence results with numerical experiments on synthetic graphons and filters, confirming the predicted behavior across graph families. Beyond validating the theory, these experiments highlight the potential practical impact of our framework as a controllable benchmark generator---by varying the underlying graphon and spectral filter, one can systematically produce graph–feature pairs spanning a wide range of homophily levels. In this sense, our approach plays a role analogous to synthetic benchmarking environments such as GraphWorld \cite{palowitch2022graphworld}, but we distinctly offer a principled mechanism grounded in graphon limits and stationary signal models.

\section{Preliminaries}
\subsection{Graph and graph signals}

A graph $G=(V,E)$ consists of two components: a set of vertices or nodes $V$, and a set of edges $E \subseteq V \times V$. 
Graphs can be categorized as either directed or undirected depending on the structure of their edge set $E$. 
A graph is undirected if and only if for any two nodes $u,v \in V$, $(u,v) \in E$ also implies $(v,u) \in E$, and both correspond to the same undirected edge. 
In this work, we restrict attention to undirected graphs.  

Let $|V|=n$ be the number of nodes and $m = |E|$ the number of edges in $G$. 
The \emph{adjacency matrix} is the $n \times n$ matrix $\mathbf{A}$ with entries
\[
\mathbf{A}_{ij} =
\begin{cases}
    1 & \text{if } (i,j) \in E, \\
    0 & \text{otherwise}.
\end{cases}
\]

The \emph{degree matrix} $\mathbf{D}$ is the $n \times n$ diagonal matrix with $\mathbf{D}_{ii} = \sum_{j=1}^n \mathbf{A}_{ij}$. 
The \emph{graph Laplacian} is then defined as $\mathbf{L} = \mathbf{D} - \mathbf{A}$.
Since $G$ is undirected, $\mathbf{A}$ and $\mathbf{L}$ are symmetric.  

A \emph{graph signal} is a vector $x \in \mathbb{R}^n$, where $x_i$ corresponds to the signal value at node $i$ \cite{shuman13-mag,ortega2018graph}. 
More generally, graph signals can have multiple features, represented as matrices $X \in \mathbb{R}^{n \times d}$, with $d$ denoting the feature dimension. In this paper we focus on a stationary signal model for graph signals originally introduced in \cite{marques2017stationary}; see also~\cite{perraudin2017}.

\begin{definition}[Stationary graph signals \cite{marques2017stationary}]
Let $e_1,\ldots, e_n \allowbreak\sim \mathcal{N}(0,I_d)$ be i.i.d. Gaussian vectors, and define the initial feature matrix
$\smash{X_0 = \frac{1}{\sqrt{d}}\,[\,e_1,\ldots,e_n\,]^\top \in \mathbb{R}^{n \times d}}$.
Let $G_n$ be a graph with n nodes sampled from any random graph model.
Denote
$\mathbf{L}_n = \mathbf{D}_n - \mathbf{A}_n$ its unnormalized Laplacian, and define the rescaled Laplacian $\mathcal{L}_n = \mathbf{L}_n/n$.  
The observed feature matrix is then generated by
\[
X = \sum_{k=0}^{K-1} a_k \, \mathcal{L}_n^k X_0 \;=\; f(\mathcal{L}_n)X_0,
\]
where $f(\cdot)$ is a so-called linear shift-invariant filter with $K-1$ coefficients. 
\end{definition}

This model is important for two main reasons.  
First, it reflects the idea that node features are not arbitrary but are \emph{propagated by the structure of the sampled graph}: the polynomial filter $f(\mathcal{L}_n)$ spreads and smooths the initial white noise features $X_0$ according to the connectivity in $G_n$.  
Second, the parameterization by powers of $\mathcal{L}_n$ allows us to control the level of feature smoothness (and thus feature homophily), linking the spectral properties of the Laplacian to the statistical properties of the features.  

\subsection{Graphons}

A \emph{graphon} is a symmetric, bounded, and measurable function $W: [0,1]^2 \to [0,1]$. 
Graphons serve both as generative models and as limits of dense graph sequences, capturing the asymptotic behavior of graphs as their size increases.  
Here we are primarily interested in the use of graphons as generative models. Given a graphon $W$, a random graph $G_n$ with $n$ nodes is sampled by drawing independent latent variables $u_1,\ldots,u_n \sim \mathrm{Unif}[0,1]$, and forming edges between nodes $i$ and $j$ independently with probability $W(u_i,u_j)$. 
This yields a random graph $G_n \sim \mathcal{G}_{W}(n)$ whose structure reflects the properties of $W$.  

An important class of graphons is given by \emph{step-function graphons}, which correspond exactly to (dense) stochastic block models (SBMs). 
Such a graphon, denoted $W_{\mathbf{P}}$, is defined as
\[
    W_{\mathbf{P}}(x,y) = P_{ij} \quad \text{if } x \in I_i,\, y \in I_j,
\]
where $\{I_1, \ldots, I_r\}$ is a partition of $[0,1]$ with $\mu(I_i) = n_i/n$. 
Step-function or SBM graphons are the canonical example of graphons of finite rank. Though the graphon acts as the limit of graphs with infinitely many nodes, it may still have low intrinsic dimensionality, which in the case of SBMs is given by the finite number of communities.

We state the following theorem, due to \cite{vizuete2020laplacian}, proving convergence of the Laplacian spectrum of $G_n \sim \mathcal{G}_{W}(n)$.

\begin{figure*}[t]
    \centering
    \includegraphics[width=0.48\textwidth,height=4.8cm]{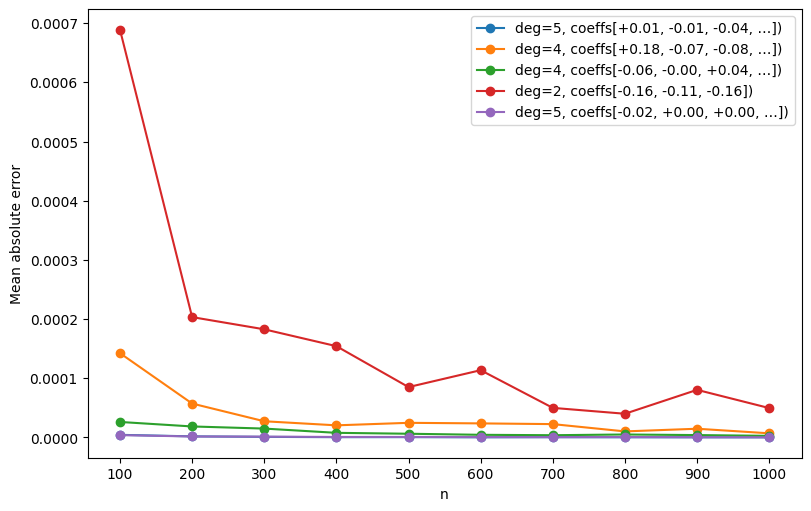}
    \hspace{1cm}
    \includegraphics[width=0.42\textwidth,height=4.8cm]{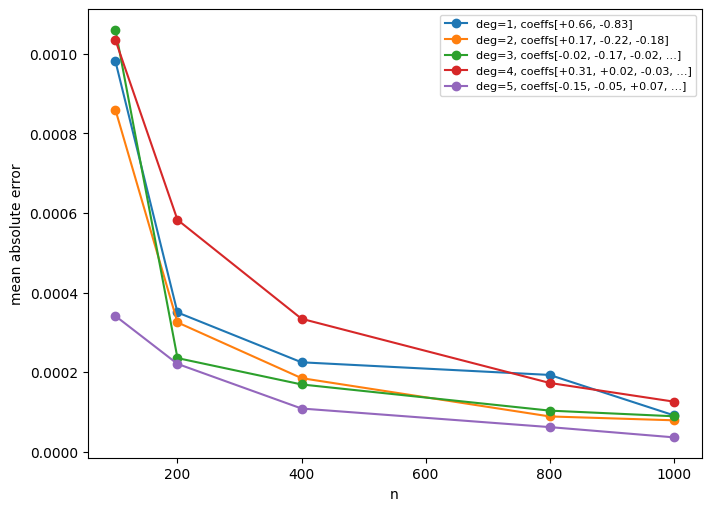}
    \caption{(Left) Validation of Proposition~\ref{prop:h_G_concentration}: 
    concentration of $h_{G_n}$ around its expected limit. (Right) Validation of Theorem~\ref{thm:as-conv-hGn}: convergence of $h_{G_n}$ 
    to its graphon limit.}
    \label{fig:theorem}
\end{figure*}

\begin{theorem}[Laplacian spectral convergence for $G_n \sim \mathcal{G}_{W}(n)$ \cite{vizuete2020laplacian}]
Let $G_n \sim \mathcal{G}_{W}(n)$ be sequence of random graphs. Denote its adjacency and degree matrices by $\mathbf{A}_n$ and $\mathbf{D}_n$, respectively, and let $\mathbf{L}_n = \mathbf{D}_n - \mathbf{A}_n$ be the (unnormalized) graph Laplacian. 
Define the rescaled Laplacian as $\mathcal{L}_n = \frac{\mathbf{L}_n}{n}.$
Then, as $n \to \infty$, the empirical spectral distribution of $\mathcal{L}_n$ converges weakly (almost surely) 
to the spectral distribution of the graphon Laplacian operator
\[
(\mathcal{L}_{W} f)(x) = f(x) - \int_{0}^1 W(x,y) f(y)\, dy, 
\quad f \in L^2([0,1]).
\]
In particular, the eigenvalues of $\mathcal{L}_n$ converge to the eigenvalues of $\mathcal{L}_{W}$.
\end{theorem}

This key result underscores many other important results such as the convergence of spectral operators parametrized by the graph Laplacian, like graph filters \cite{segarra15-interp,du2018graph} and graph neural networks \cite{gama18-gnnarchit,defferrard17-cnngraphs,kipf17-classifgcnn}, and the convergence of graph sampling sets \cite{lepoincare,parada2025sampling}. We will leverage it to prove consistency of the feature heterophily of graphs sampled from a graphon, which we discuss next.

\subsection{Feature heterophily}

Given a graph with a signal corresponding to a $d$-dimensional feature vector per node, its feature heterophily is defined as follows \cite{li2024graph}.

\begin{definition}[Feature heterophily]
Given a graph $G_n=(V,E)$ with feature matrix $X \in \mathbb{R}^{n \times d}$, 
its feature heterophily is defined as
\begin{equation}
h_{G_n} = \frac{1}{n}\operatorname{Tr}(\mathcal{L}_n XX^\top),
\label{eq:homophily_trace}
\end{equation}
where $\mathcal{L}_n = \mathbf{L}_n/n$ and $\mathbf{L}_n = \mathbf{D}_n - \mathbf{A}_n$ is the unnormalized Laplacian. 
\end{definition}

The notion of feature heterophily generalizes to arbitrary data the usual definition of homophily as a measure of the connection strength between nodes sharing the same categorical labels.
To see this, note that $\mathbf{L}_n$ can be written equivalently as $\mathbf{L}_n = \bbB^\top \bbB$, where $\bbB \in \mathbb{R}^{|E|\times n}$ is the signed incidence matrix of the graph. 
Substituting this into \eqref{eq:homophily_trace}, we obtain $\smash{h_{G_n} =  \frac{1}{n^2} \|\bbB X\|_F^2}$.
Since each row of $\bbB X$ corresponds to the difference $X_i - X_j$ for an edge $(i,j) \in E$, this expression simplifies to
\begin{equation} \label{eq:homophily_equiv}
h_{G_n} = \frac{1}{2n^2} \sum_{(i,j)\in E} \|X_i - X_j\|_2^2,
\end{equation}
recovering an expression that depends on the dissimilarity of the signal values at the end-nodes of every edge.

Smaller values of the feature heterophily $h_{G_n}$ correspond to \emph{homophily}---connected nodes have similar features---while larger values correspond to \emph{heterophily}---connected nodes have dissimilar features. 
Hence, $h_{G_n}$ quantifies the alignment between the graph topology and the feature variation over the nodes.  

\section{Theoretical Results}

We consider stochastic graphs drawn from a graphon model, augmented with node features generated by a stationary signal model, and analyze how their interaction with the underlying graph shapes the feature heterophily of the network. Then, we study the consistency of feature heterophily across graphs produced by this combined model, supporting its use as a generative framework for controllable feature heterophily.

\noindent \textbf{Heterophily concentration for fixed $G_n$.} Given a fixed graph $G_n \sim \mathcal{G}_{W}(n)$ sampled from a graphon $W$, our first result establishes the concentration of the feature heterophily under the stationary signal model. Before stating it, we require a mild assumption on the number of features $d$.

\begin{assumption}[Polynomial regime]\label{assum:prop}
The number of nodes $n$ and the feature dimension $d$ grow proportionally such that
$d^{1/\alpha} \;\leq\; n \;\leq\; d^{\alpha}$,
for some fixed $\alpha > 1$.
\end{assumption}

The polynomial regime $d^{1/\alpha}\leq n\leq d^{\alpha}$ guarantees that
the feature dimension and the number of nodes are comparable in scale,
leading to deviations that shrink on the order of $1/\sqrt{N}$, up to logarithmic factors.

\begin{proposition}[Concentration of feature heterophily]\label{prop:h_G_concentration}
Suppose Assumption~\ref{assum:prop} holds.  
Let $\mathcal{L}_n$ be the rescaled Laplacian of a random graph 
(such as one sampled from a graphon model), independent of the Gaussian initialization $X_0$.  
Let $N=\max\{n,d\}$. Then, for any $\epsilon > 0$, there exist a universal constant $C>0$ 
and constants $C_P,c_\alpha>0$ depending only on $\alpha$ and $f$ such that
\begin{align*}
\mathbb{P}\!\left(
\big|h_{G_n} - n^{-1}\operatorname{Tr}(f(\mathcal{L}_n)\mathcal{L}_n f(\mathcal{L}_n))\big|
\geq \tfrac{C_P(\log N)^{1+\epsilon}}{\sqrt{N}}
\right) \\
\leq C N^{-c_\alpha(\log N)^\epsilon}.
\end{align*}
\end{proposition}
\begin{proof}
See \href{https://github.com/caltdreamer/ICASSP-2026-HAOYU/blob/e030a47f23cb4cb666f3941ec59464178311994a/Proof_of_Proposition.pdf}{[LINK]}.
\end{proof}
Proposition~\ref{prop:h_G_concentration} has two main implications. First, it shows that the random heterophily score $h_{G_n}$,
which depends on the Gaussian initialization $X_0$, 
is tightly concentrated around its expectation\footnote{Here the expectation is taken over the randomness in $X_0$, while the graph $G_n$
(and thus $\mathcal{L}_n$) stays fixed. The independence between $G_n$ and $X_0$ is crucial, as it ensures that deviations of $h_{G_n}$ come only from Gaussian randomness, which can be controlled 
using matrix concentration arguments.}, which is given by the deterministic spectral quantity
$\mu_n = \frac{1}{n}\operatorname{Tr}\big(f(\mathcal{L}_n)\mathcal{L}_n f(\mathcal{L}_n)\big)\text{.}$
Thus, the observed feature heterophily is essentially determined by the graph spectrum
and the filter $f$ in the stationary model. 

Second, as $n \to \infty$, Proposition~\ref{prop:h_G_concentration} implies (in soft-$O$ notation) $\smash{|h_{G_n}-\mu_n|
= \tilde{O}(\frac{1}{\sqrt{N}})}$ with high probability,
where $\tilde{O}$ hides polylogarithmic factors. In fact, we can say something stronger: since the deviation probabilities in Proposition~\ref{prop:h_G_concentration} 
are summable in $N$, the Borel--Cantelli lemma implies that
$\lim_{N\to\infty} \allowbreak h_{G_n}-\mu_n = 0$ almost surely. 
Therefore, on large-scale graphs $h_{G_n}$ is eventually pinned to its deterministic
spectral counterpart $\mu_n$, providing a rigorous justification for analyzing homophily
through the graphon limit. 
Since the convergence is not only in probability but also almost surely, the concentration result is essentially tight and robust for asymptotic analysis.

\noindent \textbf{Heterophily convergence as $n \to \infty$.} Next, we study the asymptotic behavior of the random quantity $\mu_n$, which depends on the graphon model $W$, to establish feature heterophily convergence under the combined graphon and stationary signal model.
From now on, we denote $\lambda_1^{(n)},\dots,\lambda_n^{(n)}$ the non-decreasing Laplacian eigenvalues of $G_n\sim\mathcal{G}_{W}(n)$. We implicitly assume that the number of features $d(n)$ is a function of $n$ that satisfies the proportional regime from Assumption \ref{assum:prop}, and further, that the graphon $W$ is Lipschitz as in Assumption \ref{assum:W-Lipschitz} below.

\setcounter{assumption}{1} 
\begin{assumption}[Lipschitz graphon]\label{assum:W-Lipschitz}
The graphon $W:[0,1]^2\!\to[0,1]$ is symmetric and $L$--Lipschitz in the product metric. I.e., for all $(x,y),(x',y')\in[0,1]^2$,
\[
|W(x,y)-W(x',y')|
\le L\big(|x-x'|+|y-y'|\big).
\]
\end{assumption}
\begin{theorem}[Almost-sure convergence of feature heterophily $h_{G_n}$]\label{thm:as-conv-hGn}
Under Assumptions \ref{assum:prop} and \ref{assum:W-Lipschitz} and the hypotheses of 
Proposition~\ref{prop:h_G_concentration}, we have
\[
h_{G_n}\;\xrightarrow[n\to\infty]{\mathrm{a.s.}}\;
\int_0^1 \delta(x)\,f\!\big(\delta(x)\big)^2\,dx ,
\]
where $\delta(x)=\int_0^1 W(x,y)\,dy$ is the degree function of $W$.
\end{theorem}
\begin{proof}[Proof outline of Theorem~\ref{thm:as-conv-hGn}]
The proof relies on the following lemma proved in the supplementary material of the extended version 
\href{https://github.com/caltdreamer/ICASSP-2026-HAOYU/blob/b3349bd6642f3b4b6bd6dbfff000479b125b36c7/Proof_of_Lemma_1.pdf}{[LINK]}.

\begin{lemma}\label{lemma:spectrumconverge}
For any polynomial filter $P$,
\[
\frac{1}{n}\sum_{i=1}^n P(\lambda_i^{(n)})
\xrightarrow[n\to\infty]{\mathrm{a.s.}}
\int_0^1 P\big(\delta(x)\big)\,dx .
\]
\end{lemma}
Using this lemma, Theorem~\ref{thm:as-conv-hGn} follows immediately by choosing
$P(x)=x\,f(x)^2$, applying the concentration result from
Proposition~\ref{prop:h_G_concentration}, and observing that $$\frac{1}{n}\operatorname{Tr}\!\big(f(\mathcal{L}_n)\,\mathcal{L}_n\,f(\mathcal{L}_n)\big)
= \frac{1}{n}\sum_{i=1}^n \lambda_i^{(n)}\,f(\lambda_i^{(n)})^2 .$$
\end{proof}

By establishing convergence of $h_{G_n}$ on sequences of graphs converging to a graphon, Theorem \ref{thm:as-conv-hGn} implies consistency of the feature heterophily across graphs and feature vectors generated from the combined graphon and stationary signal model. This provides theoretical support for using this combined model as a principled generative framework for graphs and graph signals that allows explicit control of homophily, which is key to several information processing problems on graphs, including node-level tasks in graph machine learning \cite{zhu2020beyond} and network topology inference \cite{GM6}.

Another key takeaway from this theorem is that the limiting feature heterophily, besides depending on the function $f$ (a parameter of the stationary signal model), depends on the graphon --not through the function  $W$ or its eigenvalues-- but through its degree profile $\delta$. This is a consequence of two properties of our combined model. First, stationarity implies the ability to jointly diagonalize the Laplacian and the feature covariance, so that $h_{G_n}$ can be computed from the product of their eigenvalues. Second, due to the inherent edge density associated with graphons, the Laplacian eigenspectrum is localized around the node degrees \cite{vizuete2020laplacian}. 

Finally, we remark that by the spectral theorem applied to the matrix $\mathcal{L}$, the filter $f$ is only evaluated in the range of the spectrum of $\mathcal{L}$, which is in $[0,2]$. Thus, leveraging the fact that any continuous function can be written as a sum of polynomials uniformly over the closed interval, Theorem \ref{thm:as-conv-hGn} can readily be extended to continuous filters over the closed interval $[0,2]$. 

We conclude with two examples particularizing Theorem \ref{thm:as-conv-hGn} to an Erdős--Rényi (ER) and an SBM graphon respectively.

\begin{corollary}[ER graphon]
Let $W(x,y)=p$ for some $p\in(0,1)$. Then $\delta(x)=p$ for all $x\in[0,1]$, and 
\[
h_{G_n}\xrightarrow[n\to\infty]{\mathrm{a.s.}} p\, f(p)^2.
\]
\end{corollary}

\begin{corollary}[SBM graphon]
Consider an $r$--block SBM graphon with block proportions 
$\alpha_1,\ldots,\alpha_r$ and probability matrix $[P_{ij}]$. 
The degree function is piecewise constant:
\[
\delta(x)=\delta_i:=\sum_{j=1}^r \alpha_j P_{ij}, 
\quad \text{for } x\in I_i.
\]
Then, $h_{G_n}\xrightarrow[n\to\infty]{\mathrm{a.s.}}
\sum_{i=1}^r \alpha_i\, \delta_i\, f(\delta_i)^2$.
\end{corollary}

\section{Numerical results}

In this section, we provide empirical validation of the main theoretical results developed above. 
Specifically, we verify both the concentration result (Proposition~\ref{prop:h_G_concentration}) 
and the almost-sure convergence of the feature heterophily (Theorem~\ref{thm:as-conv-hGn}).  

\subsection{Experimental Setup}
We sample graphs from a fixed Lipschitz graphon and generate node features according to 
the stationary signal model. For each graph size $n$, we compute the empirical heterophily 
score $h_{G_n}$ and compare it with the theoretical targets described in our results.  
The validation is performed over multiple trials, with random polynomial filters $f$ of varying 
degrees. All experiments use independent graph samples and Gaussian noise features.  

\subsection{Results}
Figure~\ref{fig:theorem} (left) illustrates the empirical validation of the concentration 
in Proposition~\ref{prop:h_G_concentration}. The plot shows that the deviation 
between $h_{G_n}$ and its expectation shrinks as $n$ grows, consistent with 
the theoretical $O\!\left(\tfrac{(\log N)^{1+\epsilon}}{\sqrt{N}}\right)$ bound.

Figure~\ref{fig:theorem} (right) validates Theorem~\ref{thm:as-conv-hGn}, showing that the 
heterophily score $h_{G_n}$ converges almost surely to the deterministic integral 
$\int_0^1 \delta(x) f(\delta(x))^2 dx$ as $n\to\infty$.  
The error decays rapidly with $n$, in line with the theoretical prediction and the 
Borel--Cantelli argument ensuring almost-sure convergence.  
These experiments confirm that both the concentration bound and the almost-sure convergence 
result hold in practice, thus supporting the theoretical framework.


\bibliographystyle{IEEEbib}
\bibliography{myIEEEabrv,bib_cumulative}


\section{Appendix}
\label{sec:appendix}

\subsection{Proof of Proposition~\ref{prop:h_G_concentration}}

\begin{proof}
The proof uses the Hanson--Wright inequality. Throughout we condition on $\mathcal{L}$ (independent of $X_0$), apply concentration, and then remove the conditioning at the end. Set
\[
A:=f(\mathcal{L})\,\mathcal{L}\,f(\mathcal{L}),
\]
which is positive semidefinite because $\mathcal{L}\succeq 0$ and $f(\mathcal{L})$ is symmetric. There exists a constant $C_A>0$ (depending only on $f$ and a uniform bound on $\|\mathcal{L}\|_{\mathrm{op}}$) such that
\[
\|A\|_{\mathrm{op}}\le C_A, \qquad \|A\|_{F}\le \sqrt{n}\,C_A.
\]

\paragraph*{Case 1: $n\ge d$ (so $N=n$).}
We write
\begin{align}
h_G
&= \tfrac{1}{nd}\operatorname{Tr}(X_0^\top A X_0) \\
&= \tfrac{1}{nd}\sum_{a=1}^d (X_0\mathbf{e}_a)^\top A (X_0\mathbf{e}_a) \\
&= \tfrac{1}{n}\operatorname{Tr}(A) + \Delta(n). \nonumber
\end{align}

where
\[
\Delta(n) := \frac{1}{d}\sum_{a=1}^d E_a,\qquad
E_a := \frac{1}{n}\Big( (X_0\mathbf{e}_a)^\top A (X_0\mathbf{e}_a) - \operatorname{Tr}(A) \Big).
\]
Since $X_0\mathbf{e}_a\stackrel{\text{i.i.d.}}{\sim} \mathcal{N}(0,I_n)$ and is independent of $\mathcal{L}$, Hanson--Wright for $\gamma\sim \mathcal{N}(0,I_n)$ yields, for any $t>0$,
\[
\mathbb{P}\!\left( |E_a| \ge t \,\middle|\, \mathcal{L} \right)
\;\le\; C \exp\!\Big(-c \min\!\big\{ \tfrac{n^2 t^2}{\|A\|_F^2}, \tfrac{nt}{\|A\|_{\mathrm{op}}} \big\}\Big).
\]

Choose $t=\dfrac{C_P(\log N)^{1+\epsilon}}{\sqrt{N}}=\dfrac{C_P(\log n)^{1+\epsilon}}{\sqrt{n}}$. Then
\[
\frac{n^2 t^2}{\|A\|_F^2} \gtrsim (\log n)^{2+2\epsilon},
\qquad
\frac{nt}{\|A\|_{\mathrm{op}}} \gtrsim (\log n)^{1+\epsilon},
\]
so
\[
\mathbb{P}\!\left( |E_a| \ge \frac{C_P(\log n)^{1+\epsilon}}{\sqrt{n}} \,\middle|\, \mathcal{L} \right)
\le C\exp\!\left(-c(\log n)^{1+\epsilon}\right).
\]
Since $|\Delta(n)|\le \max_{1\le a\le d} |E_a|$, by a union bound over $a=1,\dots,d$ (with $d\le n=N$),
\begin{align}
\mathbb{P}\!\left( |\Delta(n)| \ge \tfrac{C_P(\log n)^{1+\epsilon}}{\sqrt{n}} \,\middle|\, \mathcal{L} \right)
&\le C\, d\, \exp\!\big(-c(\log n)^{1+\epsilon}\big) \\
&= C\, n^{-c_\alpha(\log n)^\epsilon}. \nonumber
\end{align}

Unconditioning preserves the bound:
\[
\mathbb{P}\!\left(\left|h_G-\frac{1}{n}\operatorname{Tr}(A)\right|\ge \frac{C_P(\log n)^{1+\epsilon}}{\sqrt{n}}\right)
\le C\, n^{-c_\alpha(\log n)^\epsilon}.
\]

\paragraph*{Case 2: $d>n$ (so $N=d$).}
Since $\mathcal{L}\succeq 0$, write $\mathcal{L}=\mathcal{L}^{1/2}\mathcal{L}^{1/2}$. Then
\begin{align}
h_G
&= \tfrac{1}{nd}\sum_{a=1}^d (X_0\mathbf{e}_a)^\top f(\mathcal{L})\mathcal{L}^{1/2}\mathcal{L}^{1/2}f(\mathcal{L})\,X_0\mathbf{e}_a \\
&= \tfrac{1}{nd}\sum_{a=1}^d\sum_{i=1}^n \big(\mathbf{e}_i^\top \mathcal{L}^{1/2} f(\mathcal{L}) X_0\mathbf{e}_a\big)^2 \\
&= \tfrac{1}{n}\sum_{i=1}^n \frac{\|f(\mathcal{L})\mathcal{L}^{1/2}\mathbf{e}_i\|_2^2}{d}\sum_{a=1}^d \gamma_a^2 \\
&= \tfrac{1}{n}\sum_{i=1}^n \|f(\mathcal{L})\mathcal{L}^{1/2}\mathbf{e}_i\|_2^2 + \Delta(d) \\
&= \tfrac{1}{n}\operatorname{Tr}\!\big(\mathcal{L}^{1/2}f(\mathcal{L})f(\mathcal{L})\mathcal{L}^{1/2}\big) + \Delta(d) \\
&= \tfrac{1}{n}\operatorname{Tr}(A) + \Delta(d). \nonumber
\end{align}

Here $\gamma\sim \mathcal{N}(0,I_d)$ and
\[
\Delta(d) := \tfrac{1}{n}\sum_{i=1}^n E_i,\qquad
E_i := \tfrac{\|f(\mathcal{L})\mathcal{L}^{1/2}\mathbf{e}_i\|_2^2}{d}\,\Big(\gamma^\top I_d\gamma - \operatorname{Tr}(I_d)\Big).
\]

Using uniform bounds, $\|f(\mathcal{L})\mathcal{L}^{1/2}\mathbf{e}_i\|_2^2\le C_A$, hence
\begin{align}
|E_i|
&\le \tfrac{C_A}{d}\, \Big|\gamma^\top I_d\gamma - \operatorname{Tr}(I_d)\Big|. \nonumber
\end{align}

Apply Hanson--Wright with $A=I_d$ ($\|I_d\|_F=\sqrt d$, $\|I_d\|_{\mathrm{op}}=1$) and choose
\[
t=\tfrac{C_P(\log N)^{1+\epsilon}}{\sqrt{N}}
=\tfrac{C_P(\log d)^{1+\epsilon}}{\sqrt{d}}.
\]

Then
\begin{align}
\mathbb{P}\!\left( |E_i|\ge t \,\middle|\, \mathcal{L} \right)
&\le \mathbb{P}\!\left( \big|\gamma^\top I_d\gamma-\operatorname{Tr}(I_d)\big|\ge (\log d)^{1+\epsilon}\sqrt d \right) \\
&\le C\exp\!\left(-c(\log d)^{1+\epsilon}\right). \nonumber
\end{align}

A union bound over $i=1,\dots,n$ (with $n<d=N$) gives
\begin{align}
\mathbb{P}\!\left( |\Delta(d)| \ge \tfrac{C_P(\log d)^{1+\epsilon}}{\sqrt{d}} \,\middle|\, \mathcal{L} \right)
&\le C\, n\, \exp\!\left(-c(\log d)^{1+\epsilon}\right) \\
&= C\, d^{-c_\alpha(\log d)^\epsilon}. \nonumber
\end{align}

Unconditioning yields
\begin{align}
\mathbb{P}\!\left(\Big|h_G-\tfrac{1}{n}\operatorname{Tr}(A)\Big|\ge \tfrac{C_P(\log d)^{1+\epsilon}}{\sqrt{d}}\right)
&\le C\, d^{-c_\alpha(\log d)^\epsilon}. \nonumber
\end{align}
\end{proof}
\subsection{Proof of Lemma~\ref{lemma:spectrumconverge}}

\begin{proof}
By linearity, it suffices to consider $P(t)=t^m$ with $m\ge 1$ (the case $m=0$ is trivial).

For any fixed $m\ge 1$,
\begin{align}
\frac{1}{n}\sum_{i=1}^n \big(\lambda_i^{(n)}\big)^{m}
&= \frac{1}{n}\Tr(\mathcal{L}^{\,m}) \\
&= \frac{1}{n^{m+1}}\Tr\big((D-A)^m\big). \label{eq:DA-expand}
\end{align}
Expanding $(D-A)^m$, the all-$D$ word contributes
\[
\frac{1}{n^{m+1}}\Tr(D^m)=\frac{1}{n}\sum_{i=1}^n\Big(\frac{d_i}{n}\Big)^{\!m}.
\]

\begin{claim}[]\label{lem:mixed-word}
If $w(D,A)$ is a word of length $m$ containing at least one $A$, then
\[
\big|\Tr\,w(D,A)\big|\le n^{m}.
\]
\end{claim}

\begin{proof}
By cyclicity of trace, write
\[
w(D,A)=D^{\alpha_1}AD^{\alpha_2}A\cdots AD^{\alpha_r},
\]
where $r\ge 1$ is the number of $A$'s and $\sum_{s=1}^r\alpha_s=m-r$. Then
\[
\Tr\,w(D,A)=\sum_{i_1,\dots,i_r=1}^n
\Big(d_{i_1}^{\alpha_1}A_{i_1i_2}d_{i_2}^{\alpha_2}\cdots A_{i_ri_1}d_{i_1}^{\alpha_r}\Big).
\]
Each summand has absolute value at most $n^{m-r}$ since $d_i\le n$ and $A_{ij}\in\{0,1\}$, and there are $n^r$ tuples. Hence $|\Tr\,w(D,A)|\le n^m$.
\end{proof}

After dividing by $n^{m+1}$, all mixed words vanish as $O(1/n)$. Therefore
\begin{equation}\label{eq:step1}
\frac{1}{n}\sum_{i=1}^n \big(\lambda_i^{(n)}\big)^{m}
= \frac{1}{n}\sum_{i=1}^n\Big(\frac{d_i}{n}\Big)^{\!m}+o(1).
\end{equation}

Define $m_i:=\sum_{j\neq i}W(x_i,x_j)$. We claim that
\begin{equation}\label{eq:step2}
\max_{1\le i\le n}\Big|\frac{d_i}{n}-\delta(x_i)\Big|\xrightarrow{\mathrm{a.s.}}0.
\end{equation}
Conditional on $X=(x_j)$, $d_i$ is a sum of independent Bernoullis with mean $m_i$. Hoeffding and a union bound yield
\[
\mathbb{P}\!\left(\max_i\Big|\tfrac{d_i}{n}-\tfrac{m_i}{n}\Big|>\varepsilon\,\Big|\,X\right)\le 2n\,e^{-2\varepsilon^2 n}.
\]
The RHS is summable and independent of $X$, hence by Borel--Cantelli,
\[
\max_i\Big|\tfrac{d_i}{n}-\tfrac{m_i}{n}\Big|\to 0\quad\text{a.s.}
\]

Set
\[
Z_n:=\sup_{x\in[0,1]}\left|\frac{1}{n}\sum_{j=1}^n W(x,x_j)-\delta(x)\right|.
\]
Fix $\varepsilon>0$, let $\eta:=\varepsilon/(4L)$, and let $\mathcal G$ be a grid on $[0,1]$ with mesh $\eta$ and size $|\mathcal G|\le 3/\eta$. For each $x$ choose $u\in\mathcal G$ with $|x-u|\le\eta$.

\begin{claim}[Lipschitzness of $\delta$]\label{lem:delta-lip}
$\delta$ is $L$--Lipschitz: $|\delta(u)-\delta(x)|\le L|u-x|$.
\end{claim}

\begin{proof}
\begin{align}
|\delta(u)-\delta(x)|
&= \left|\int_0^1 \big(W(u,y)-W(x,y)\big)\,dy\right| \\
&\le \int_0^1 L|u-x|\,dy \\
&= L|u-x|.
\end{align}

\end{proof}

By triangle inequality and the $L$--Lipschitz property of $W$ in its first argument and of $\delta$,
\begin{align}
\left|\frac{1}{n}\sum_{j=1}^n W(x,x_j)-\delta(x)\right|
&\le \left|\frac{1}{n}\sum_{j=1}^n\big(W(x,x_j)-W(u,x_j)\big)\right| \\
&\quad +\left|\frac{1}{n}\sum_{j=1}^n W(u,x_j)-\delta(u)\right| \\
&\quad +|\delta(u)-\delta(x)| \\
&\le L|x-u| \\
&\quad +\left|\frac{1}{n}\sum_{j=1}^n W(u,x_j)-\delta(u)\right| \\
&\quad +L|x-u| \\
&\le \left|\frac{1}{n}\sum_{j=1}^n W(u,x_j)-\delta(u)\right| \\
&\quad +2L\eta. \label{eq:triangle-final}
\end{align}

Thus
\begin{align}
Z_n
&\le \max_{u\in\mathcal G}\left|\frac{1}{n}\sum_{j=1}^n W(u,x_j)-\delta(u)\right|+\frac{\varepsilon}{2}. \label{eq:Zn-grid}
\end{align}

For fixed $u\in\mathcal G$, the variables $W(u,x_j)\in[0,1]$ are i.i.d. with mean $\delta(u)$, so Hoeffding yields
\[
\mathbb{P}\!\left(\Big|\tfrac{1}{n}\sum_{j=1}^n W(u,x_j)-\delta(u)\Big|>\tfrac{\varepsilon}{2}\right)\le 2e^{-\frac{\varepsilon^2}{2}n}.
\]
A union bound over $\mathcal G$ gives
\begin{align}
\mathbb{P}(Z_n>\varepsilon)
&\le \frac{6}{\eta}\,e^{-\frac{\varepsilon^2}{2}n} \\
&= \frac{24L}{\varepsilon}\,e^{-\frac{\varepsilon^2}{2}n}, \label{eq:Zn-summable}
\end{align}
which is summable; hence $Z_n\to 0$ a.s. Moreover,
\[
\Big|\tfrac{m_i}{n}-\delta(x_i)\Big|\le Z_n+\tfrac{1}{n}\to 0\quad\text{a.s.}
\]
Combining previous results yields \eqref{eq:step2}. Finally, since $t\mapsto t^m$ is $m$--Lipschitz on $[0,1]$,
\begin{equation}\label{eq:step3}
\left|\frac{1}{n}\sum_{i=1}^n\Big(\tfrac{d_i}{n}\Big)^{\!m}-\frac{1}{n}\sum_{i=1}^n \delta(x_i)^m\right|
\le m\,\max_i\Big|\tfrac{d_i}{n}-\delta(x_i)\Big|\xrightarrow{\mathrm{a.s.}}0.
\end{equation}
By the strong law of large numbers,
\begin{equation}\label{eq:step4}
\frac{1}{n}\sum_{i=1}^n \delta(x_i)^m \xrightarrow{\mathrm{a.s.}} \int_0^1 \delta(x)^m\,dx.
\end{equation}
We conclude from \eqref{eq:step1}, \eqref{eq:step3}, and \eqref{eq:step4}, and triangle inequality,
\[
\frac{1}{n}\sum_{i=1}^n \big(\lambda_i^{(n)}\big)^{m} \xrightarrow{\mathrm{a.s.}} \int_0^1 \delta(x)^m\,dx.
\]
By linearity, the claim follows for all polynomials $P$.
\end{proof}

\end{document}